%% file: hyper.tex
\documentclass[11pt]{article}
\usepackage[utf8]{inputenc}
\usepackage[margin=1.15in]{geometry}
\usepackage[colorlinks,linkcolor=blue,citecolor=blue]{hyperref}

\usepackage{tikz}
\usepackage{algorithm,algorithmic}
\usepackage{amsfonts}
\usepackage{nicefrac}

\include{header}

\renewcommand{\epsilon}{\varepsilon}

\title{Hyperparameter Optimization: A  Spectral Approach}

\author{
	Elad Hazan \\
	Department of Computer Science\\
	Princeton University \\
	\texttt{ehazan@cs.princeton.edu}\\
	\and
	Adam Klivans \\
	Department of Computer Science\\
	University of Texas at Austin \\
	\texttt{klivans@cs.utexas.edu}
	\and
	Yang Yuan \\
	Department of Computer Science\\
	Cornell University \\
	\texttt{yangyuan@cs.cornell.edu}
}
\date{}

\begin{document}
\maketitle
\thispagestyle{empty}

\begin{abstract}
 We give a simple, fast algorithm for hyperparameter optimization
 inspired by techniques from the analysis of Boolean functions.  We
 focus on the high-dimensional regime where the canonical example is
 training a neural network with a large number of hyperparameters.
 The algorithm --- an iterative application of compressed sensing
 techniques for orthogonal polynomials --- requires only uniform
 sampling of the hyperparameters and is thus easily parallelizable.
 
 Experiments for training deep neural networks on Cifar-10 show that compared to
 state-of-the-art tools (e.g., Hyperband and Spearmint), our algorithm
 finds significantly improved solutions, in some cases better than what
 is attainable by hand-tuning.  In terms of overall running time
 (i.e., time required to sample various settings of hyperparameters
 plus additional computation time), we are at least an order of
 magnitude faster than Hyperband and 
 Bayesian
 Optimization.  We also outperform Random Search $8\times$.
  
 Our method is inspired by provably-efficient algorithms for learning
 decision trees using the discrete Fourier transform.  We obtain
 improved sample-complexty bounds for learning decision trees while
 matching state-of-the-art bounds on running time (polynomial and
 quasipolynomial, respectively). 

\end{abstract}

\setcounter{page}{0}
\newpage

\section{Introduction}

Large scale machine learning and optimization systems usually involve a large number of free parameters for the user to fix according to their application. A timely example is the training of deep neural networks for a signal processing application: the ML specialist needs to decide on an architecture, depth of the network, choice of connectivity per layer (convolutional, fully-connected, etc.), choice of optimization algorithm and recursively choice of parameters inside the optimization library itself (learning rate, momentum, etc.).  

Given a set of hyperparameters and their potential assignments, the naive practice is to search through the entire grid of parameter assignments and pick the one that performed the best, a.k.a. ``grid search". As the number of hyperparameters increases, the number of possible assignments increases exponentially and a grid search becomes quickly infeasible. It is thus crucial to find a method for automatic tuning of these parameters. 

This auto-tuning, or finding a good setting of these parameters, is now referred to as
hyperparameter optimization (HPO), or simply automatic machine learning (auto-ML).
For continuous hyperparameters, gradient descent is usually the method of choice
\cite{Maclaurin15,scalablegradientbasedtuning,drmad}.  Discrete parameters, however, such as choice of architecture, number of layers, connectivity and so forth are significantly more challenging. More formally, let 
$$ f:  \{-1,1\}^n \mapsto [0,1] $$ 
be a function mapping hyperparameter choices to test error of our model. That is, each dimension corresponds to a certain hyperparameter (number of layers, connectivity, etc.), and for simplicity of illustration we encode the choices for each parameter as binary numbers $\{-1,1\}$. The goal of HPO is to approximate the  minimizer $x^* = \argmin_{x \in \{0,1\}^n} f(x)$ in the following setting: 
\begin{enumerate}
\item
Oracle model: evaluation of $f$ for a given choice of hyperparameters is assumed to be very expensive. Such is the case of training a given architecture of a huge dataset.  

\item
Parallelism is crucial: testing several model hyperparameters in parallel is entirely possible in cloud architecture, and dramatically reduces overall optimization time. 

\item
$f$ is structured. 
\end{enumerate} 

The third point is very important since clearly HPO is  information-theoretically hard and $2^n$ evaluations of the function are necessary in the worst case. 
Different works have considered exploiting one or more of the properties above.  The approach of Bayesian optimization \cite{bayesianOPT} addresses the structure of $f$, and assumes that a useful prior distribution over the structure of $f$ is known in advance.   Multi-armed bandit algorithms \cite{hyperband}, and 
Random Search  \cite{Bergstra12}, exploit computational parallelism very well, but do not exploit any particular structure of $f$. These approaches are surveyed in more detail later.

\subsection{Our contribution}

In this paper we introduce a new {\em spectral} approach to hyperparameter
optimization. Our main idea is to make assumptions on the structure of $f$ in the Fourier domain. Specifically we assume that $f$ can be approximated by a sparse and low degree polynomial in the Fourier basis. This means intuitively that it can be approximated well by a decision tree. 

The implication of this assumption is that we can obtain a rigorous theoretical guarantee: approximate minimization of $f$ over the boolean hypercube with 
{\bf 
function evaluations only 
linear in sparsity that can be carried out in parallel}.  
We further give improved heuristics on this basic construction and show experiments showing our assumptions are validated in practice for HPO as applied to deep learning over image datasets.

%

Thus our contributions can be listed as:
\begin{itemize}
\item
A new spectral method called {\em Harmonica} that has provable
guarantees: sample-efficient recovery if the underlying hyperparameter
objective is a sparse (noisy) polynomial and easy to implement on parallel architectures.
\item
 Improved bounds on the sample complexity of  learning noisy, size $s$ decision
 trees over $n$ variables under the uniform distribution. We observe
 that the classical sample complexity bound of $n^{O({\log (
     s/\epsilon)})}$ due to Linial et al. \cite{LMN} to quadratic in the size of the tree
 $\tilde{O}(s^{2}/\epsilon \cdot \log n)$ while matching the best known quasipolynomial bound in
        running time. 

 \item
 We demonstrate significant improvements in accuracy, sample
 complexity, and running time for deep neural net training
 experiments.  We compare ourselves to state-of-the-art solvers from
 Bayesian optimization, Multi-armed bandit techniques, and Random Search.  Projecting to even higher numbers of hyperparameters, we perform simulations that show several orders-of-magnitude of speedup versus Bayesian optimization techniques.
\end{itemize}

\subsection{Previous work}

The literature on discrete-domain HPO can be roughly divided into two: probabilistic approaches and decision-theoretic methods. In critical applications, researchers usually use a grid search over all parameter space, but that becomes quickly prohibitive as the number of hyperparameter grows.  Gradient-based methods such as \cite{Maclaurin15,scalablegradientbasedtuning,drmad,gradientbasedoptimizationofhyperparamters} are applicable only to continuous hyperparameters which we do not consider.

\paragraph{Probabilistic methods and Bayesian optimization.} Bayesian optimization (BO) algorithms  \cite{tpe,bayesianOPT,multitaskBO,inputBO,inequBO,highDim,rbfbayesian} 
tune hyperparameters by  assuming a prior distribution of the loss function, and then keep updating this prior distribution based on the new observations. Each new observation is selected according to an acquisition function, which balances exploration and exploitation such that the new observation gives us a better result, or helps gain more information about the loss function. 
The BO approach is inherently serial and difficult to parallelize, and its theoretical guarantees have thus far been limited to statistical consistency (convergence in the limit).  

\paragraph{Decision-theoretic methods.} Perhaps the simplest approach to HPO is random sampling of different choices of parameters and picking the best amongst the chosen evaluations \cite{Bergstra12}. It is naturally very easy to implement and parallelize. Upon this simple technique, researchers have tried to allocate different budgets to the different evaluations, depending on their early performance. Using adaptive resource allocation techniques found in the multi-armed bandit literature, 
Successive Halving (SH) algorithm was introduced \cite{successive}. 
Hyperband further improves SH by automatically tuning the hyperparameters in SH \cite{hyperband}.

\paragraph{Learning decision trees.}

Prior work for learning decision trees (more generally Boolean functions
that are approximated by low-degree polynomials) used the celebrated
``low-degree'' algorithm of Linial, Mansour, and Nisan \cite{LMN}.  Their
algorithm uses random sampling to estimate each low-degree Fourier
coefficient to high accuracy.  


We make use of the approach of Stobbe and Krause \cite{StobbeKrause},
who showed how to learn low-degree, sparse Boolean functions using
tools from compressed sensing (similar approaches were taken by
Kocaoglu et al. \cite{KSDK14} and Negahban and Shah \cite{NS12}).  We observe that their approach can be
extended to learn functions that are both ``approximately sparse'' (in the sense
that the $L_1$ norm of the coefficients is bounded) and ``approximately low-degree''
(in the sense that most of the $L_2$ mass of the Fourier spectrum
resides on monomials of low-degree).  This implies the first decision
tree learning algorithm with polynomial sample complexity that handles
adversarial noise.  In addition, we obtain the optimal dependence on
the error parameter $\epsilon$.

For the problem of learning {\em exactly} $k$-sparse Boolean functions
over $n$ variables, Haviv and Regev \cite{HavivR2} have recently shown
that $O(n k \log n)$ uniformly random samples suffice.  Their result
is not algorithmic but does provide an upper bound on the
information-theoretic problem of how many samples are required to
learn.  The best algorithm in terms of running time for learning
$k$-sparse Boolean functions is due to \cite{FGKP}, and requires time
$2^{\Omega(n/\log n)}$. It is based on the Blum, Kalai, and Wasserman
algorithm for learning parities with noise \cite{BKW}.

\paragraph{Techniques.} Our methods are heavily based on known results from the analysis of boolean functions as well as compressed sensing. The relevant material and literature are given in the next section.

\section{Setup and definitions}


The problem of hyperparameter optimization is that of minimizing a discrete, real-valued function, which we denote by  $f: \{-1,1\}^n \mapsto [-1,1]$  (we can handle arbitrary inputs, binary is chosen for simplicity of presentation).   

In the context of hyperparameter optimization, function
evaluation is very expensive, although parallelizable, as it
corresponds to training a deep neural net.   In contrast,
any computation that does not involve function evaluation is
considered less expensive, such as computations that require
 time $\Omega(n^d)$ for ``somewhat large'' $d$ or are subexponential (we still consider
 runtimes that are exponential in $n$ to be costly).

\subsection{Basics of Fourier analysis} \label{sec:fourier}

The reader is referred to \cite{booleananalysis} for an in depth treatment of Fourier analysis of Boolean functions.
Let $f: \X \mapsto [-1,1]$ be a function over domain $\X \subseteq \reals^n$.  Let ${\cal D}$ a probability 
distribution on ${\cal X}$.   We write  $g \equiv_\epsilon f$ and say that $f,g$ are {\bf $\eps$-close} if 
$$ \E_{x \sim \D} [(f(x) - g(x))^2] \leq \eps .$$
\begin{definition} \cite{rauhut2010compressive} 
We say a family of functions
$\psi_{1},\ldots, \psi_{N}$ ($\psi_i$ maps ${\cal X}$ to $\R$) is a {\em Random Orthonormal Family} with
respect to ${\cal D}$ if 
$$\E_\D[\psi_{i}(X) \cdot \psi_{j}(X)] = \delta_{ij}  = \mycases  {1}{ if $i = j$}{ 0}{ otherwise} .$$
\end{definition}
The expectation is taken with respect to probability distribution
${\cal D}$.  We say that the family is $K$-bounded if $\sup_{x \in {\cal X}} |\psi_{i}(x)| \leq K$ for every $i$.  Henceforth we assume $K=1$.

An important example of a random orthonormal family is the class of
parity functions with respect to the uniform distribution on
$\{-1,1\}^n$:

\begin{definition}
  A parity function on some subset of variables $S \subseteq [n]$ is
  the function $\chi_S: \{-1,1\}^n \mapsto \{-1,1\}$ where
  $\chi_{S}(x) = \prod_{i \in S} x_i $.  
\end{definition}

It is easy to see that the set of all $2^{n}$ parity functions
$\{\chi_{S}\}$, one for each $S \subseteq [n]$, form a random
orthonormal family with respect to the uniform distribution on
$\{-1,1\}^n$.  

This random orthonormal family is 
  often referred to as the Fourier basis, as it is a complete
  orthonormal basis for the class of Boolean functions with respect to
  the uniform distribution on $\{-1,1\}^n$.  More generally,
  for any $f: \{-1,1\}^n \mapsto \R$, $f$ can be uniquely represented
  in this basis as
$$ f(x) = \sum_{S \subseteq [n]} \fhat_S \chi_S (x) $$
where 
$$\hat{f}_S = \langle f , \chi_S \rangle = \E_{x \in \{-1,1\}^n } [ f(x) \chi_S(x) ] $$
is the Fourier coefficient corresponding to $S$ where $x$ is drawn
uniformly from $\{-1,1\}^n$.  We also have Parseval's identity: $\E[f^2] = \sum_{S} \fhat_S^2$. 

In this paper, we will work exclusively with the above parity basis.  Our results
apply more generally, however, to any orthogonal family of polynomials
(and corresponding product measure on $\R^n$).  For example, if we
wished to work with continuous hyperparameters, we could work with
families of Hermite orthogonal polynomials with respect to
multivariate spherical {\em Gaussian} distributions.

We conclude with a definition of low-degree, approximately sparse
(bounded $L_1$ norm) functions:

\begin{definition}[Approximately sparse function] \label{def:sparse}
	Let $\{\chi_S\}$ be the parity basis, and let ${\cal C}$ be a
        class of functions mapping $\{-1,1\}$ to $\R$.  Thus for $f
        \in {\cal C}$,  $f = \sum_{S}
	\hat{f}(S) \chi_{S}$.  	We say that:
	\begin{itemize}
	\item
	 A function $f \in C$ is {\bf $s$-sparse} if
        $L_0(f) \leq s$, ie., f has at most $s$ nonzero entries in its
        polynomial expansion.  
        \item
        $f$ is {\bf $(\epsilon,d)$-concentrated} if $\E[(f - \sum_{S, |S| \leq d} \hat{f}(S)
	\chi_{S})^2]\geq 1-\eps$.   
	\item
	 ${\cal C}$ is {\bf $(\eps, d, s)$-bounded} if for
	every $f \in {\cal C}$, $f$ is $(\epsilon, d)$-concentrated and in addition ${\cal C}$ has $L_1$ norm bounded by
	$s$, that is, for every $f
	\in {\cal C}$ we have $\sum_{S} |\hat{f}(S)| \leq s$.
\end{itemize}

\end{definition}

It is easy to see that the class of functions with bounded $L_1$ norm
is more general than sparse functions.  For example, the Boolean
$\mathsf{AND}$ function has $L_1$ norm bounded by $1$ but is not
sparse. 

We also have the following simple fact: 
\begin{fact} \label{sparsefact} \cite{MansourSurvey}
	Let $f$ be such that $L_1(f) \leq s$.  Then there exists $g$ such that
	$g$ is $s^2/\eps$ sparse and $E[(f-g)^2] \leq \epsilon$. The function
	$g$ is constructed by taking all coefficients of magnitude $\eps/s$ or
	larger in $f$'s expansion as a
	polynomial.
\end{fact}

\subsection{Compressed sensing and sparse recovery}

In the problem of {\it sparse recovery}, a learner attempts to recover a sparse vector $x \in \reals^n$ which is $s$ sparse, i.e. $\|x\|_0 \leq s$,  from an observation vector $y \in R^m$ that is assumed to equal 
$$ y = A x + e ,$$
where $e$ is assumed to be zero-mean, usually Gaussian, noise. The
seminal work of \cite{Candes,Donoho} shows how $x$ can be recovered
exactly under various conditions on the observation matrix $A \in \reals^{m \times n}$ and the noise.  The usual method for recovering the signal proceeds by solving a convex optimization problem consisting of $\ell_1$ minimization as follows (for some parameter $\lambda > 0$):
\begin{align} \label{alg:BPsimple}
& \min_{x \in \reals^n} \left\{   \| x \|_1 + \lambda \| A x - y \|_2^2 \right\} .
\end{align}
The above formulation comes in many equivalent forms (e.g., Lasso), where one of the objective parts may appear as a hard constraint.

For our work, the most relevant extension of traditional sparse
recovery is due to  Rauhut \cite{rauhut2010compressive}, who considers the problem of sparse recovery when
the measurements are evaluated according to a {\em random orthonormal
family}.  More concretely, fix $x \in \R^{n}$ with $s$ non-zero entries.
For $K$-bounded random orthonormal family ${\cal F} =
\{ \psi_{1},\ldots,\psi_{N}$\}, and $m$ independent draws
$z^{1},\ldots,z^{m}$ from corresponding distribution ${\cal D}$ define
the $m \times N$ matrix $A$ such that $A_{ij} = \psi_{j}(z^i)$.
Rauhut gives the following result for recovering sparse vectors $x$:

\begin{theorem}[Sparse Recovery for Random Orthonormal Families, 
	\cite{rauhut2010compressive}  Theorem 4.4] \label{GG}
	Given as input matrix $A\in \mathbb{R}^{m\times N}$ and vector $y$ with
	$y_{i} = Ax + e_i$ for some vector $e$ with $\|e\|_2 \leq \eta
	\sqrt{m}$, mathematical program \eqref{alg:BPsimple} finds a vector $x^{*}$ such that  (for constants $c_1$ and $c_2$)
	
	\vspace{-0.8em}
	\[ \|x - x^{*}\|_2 \leq c_1\frac{\sigma_{s}(x)_1}{\sqrt{s}} +
        c_2 \eta \]
	
	with probability $1 - \delta$ as long as, for sufficiently large constant C, 
	\[ m \geq CK^{2} \log K \cdot s \log^{3} s \cdot \log^{2}N
	\cdot \log(1/\delta) . \]

\end{theorem}

The term $\sigma_{s}(x)_1$ is equal to $\min \{\|x-z\|_1, z \text{~is
	$s$ sparse}\}$.   Recent work \cite{Bourgain,HavivRegev} has
improved the dependence on the polylog factors in the lower bound for $m$.

\section{Basic Algorithm and Main Theoretical Results}

The main component of our spectral algorithm for hyperparameter
optimization is given in Algorithm \ref{alg:basic-harmonica}. It is essentially an
extension of sparse recovery (basis pursuit or Lasso) to the
orthogonal basis of polynomials in addition to an optimization step. See Figure \ref{fig:illustration} for an illustration. We prove \SHOR's theoretical guarantee, and show how it gives rise to new theoretical results in learning from the uniform distribution. 

\begin{figure}[t]
\centering
\begin{tikzpicture}[scale=1.1]	
\definecolor{mycolor}{RGB}{51, 153, 255}
\definecolor{mygreen}{RGB}{51, 204, 51}
\definecolor{mypurple}{RGB}{153, 0, 204}

\fill[color=mycolor] (0,0) rectangle (2,3);

\node at (2.5,1.5) {
	\fontsize{20pt}{12}\selectfont
	$\times$};
\fill [mygreen] (3,0.5) rectangle (3.4,2.5);
\node[color=white] at (3.2,1.5) {
	\fontsize{15pt}{12}\selectfont
	$\alpha$};
\node[text width=3cm] at (5.2,1.7) {
	\fontsize{10pt}{10}\selectfont
	$\alpha$ has entry $\alpha_S$ ~~~ for all $|S|\leq d$ };
\node[text width=3cm] (xx) at (0.2,-0.6) {
	\fontsize{10pt}{10}\selectfont
rows corresponds to $x_i\in \{-1,1\}^n$ };
\draw[->] (xx) edge [in=180, out=140] (-0.1,1);
\node[text width=3cm] (xx) at (1.2,3.45) {
	\fontsize{10pt}{10}\selectfont
	columns $S\subseteq [n]$ for all $|S|\leq d$ };
\fill [mypurple] (-1.4,0) rectangle (-1,3);
\node at (-0.45,1.5) {
	\fontsize{20pt}{12}\selectfont
	$=$};
\node[color=white] at (-1.2,1.5) {
	\fontsize{12pt}{12}\selectfont
	$f$};
\node[text width=4cm] at (-3.2,1.7) {
	\fontsize{10pt}{10}\selectfont
entry (i) corresponds to $f(x_i)=\sum_S \alpha_S \psi_S(x_i)$, the $i$-th measurement};
\node[color=white,text width=2cm] at (1.2,1.5) {
	\fontsize{10pt}{12}\selectfont
	entry (i,S) $=\psi_S(x_i)$};
\end{tikzpicture}	
\caption{{\small Compressed sensing over the Fourier domain: \SHOR{} recovers the Fourier coefficients of a sparse low degree polynomial $\sum_S \alpha_S \Psi_S(x_i) $ from observations $f(x_i)$ of randomly chosen points $x_i \in \{-1,1\}^n$.}}
\label{fig:illustration}
\end{figure}
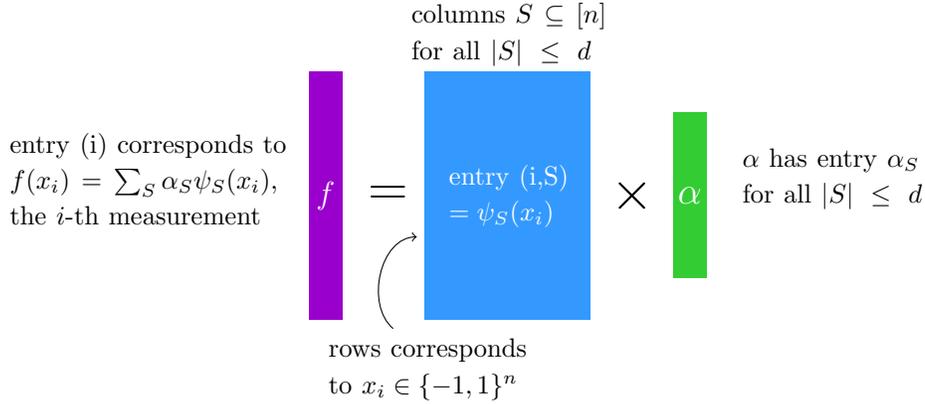

In the next section we describe
extensions of this basic algorithm to a more practical algorithm with various heuristics to improve its performance.

\begin{algorithm}[h!]
\caption{\SHOR-1}
\label{alg:basic-harmonica}
\begin{algorithmic}[1]
\STATE Input: oracle for $f$, number of samples $T$, sparsity $s$,
degree $d$,  parameter $\lambda$.
\STATE  Invoke PSR$(f,T,s,d,\lambda)$ (Procedure \ref{alg:2}) to obtain
$(g,J)$, where $g$ is a function defined on variables specified by index
set $J \subseteq [n]$.  
\STATE Set the variables in $[n] \setminus J$ to arbitrary values,  compute a minimizer $  x^\star \in  \arg \min g_i(x)$.
\RETURN $x^\star$
\end{algorithmic}
\end{algorithm}

\floatname{algorithm}{Procedure}

\begin{algorithm}[h!]
\caption{Polynomial Sparse Recovery (PSR)}
\label{alg:2}
\begin{algorithmic}[1]
\STATE Input: oracle for $f$, number of samples $T$, sparsity $s$,  degree $d$, regularization parameter $\lambda$

\STATE Query $T$ random samples: $\{  f(x_1),....,f(x_T)\} $.
\STATE Solve sparse $d$-polynomial regression over all polynomials up to degree $d$
\begin{eqnarray} \label{eqn:shalom}
\argmin _{\alpha \in \reals^{{n \choose d}}}  \left\{  \sum_{i=1}^T \left(  \sum_{|S| \leq d}   \alpha_S  \psi_S(x_i)  - f(x_i) \right)^2  + \lambda \| \alpha\|_1  \right\}
\end{eqnarray}
\STATE Let $S_1,...,S_s$ be the indices of the largest coefficients of $\vec{\alpha}$. Let $g$ be the polynomial
$$ g(x) =  \sum_{i \in [s]}   \alpha_{S_i}  \psi_{S_i}(x)   $$
\RETURN $g$ and $J = \cup_{i=1}^{s} S_{i}$
\end{algorithmic}
\end{algorithm}
\floatname{algorithm}{Algorithm}

\begin{theorem}[Noiseless recovery] \label{thm:main-harmonica1}
Let $\{\psi_{S}\}$ 
be a $K$-bounded orthonormal polynomial basis for distribution ${\cal D}$.
Let $f:\R^n \mapsto \R$  be a $(0,d,s)$-bounded function as per definition \ref{def:sparse} with respect to the basis $\psi_S$.
Then Algorithm \ref{alg:basic-harmonica}, in time $n^{O(d)}$ and sample complexity $T = \tilde{O}(K^{2} s\cdot  d \log n)$,   returns $x^\star$ such that
$$ x^\star \in \argmin f(x)$$
\end{theorem}

This theorem, and indeed most of the theoretical results of this
paper, follow from the main recovery properties of Procedure
\ref{alg:2}. Our main technical lemma follows the same outline of the
compressed sensing result due to Stobbe and Krause \cite{StobbeKrause}
but with a generalization to functions that are {\em approximately}
sparse and low-degree:

\begin{lemma}[Noisy recovery] \label{thm:main}
Let $\{\psi_{S}\}$ 
be a $K$-bounded orthonormal polynomial basis for distribution ${\cal D}$.
 Let  $f:\R^n \mapsto \R$ be a $(\eps/4, d, s)$-bounded as per definition \ref{def:sparse} with respect to the basis $\psi_S$.
Then Procedure \ref{alg:2} finds a function $g \equiv_\epsilon f$  in time $O( n^d)$ and sample
	complexity $T = \tilde{O}(K^2 s^2 / \epsilon \cdot d\log n)$.
\end{lemma}

\ignore{
\noindent {\bf Remark:} In this paper we focus on distributions ${\cal D}$ that are
products of {\em discrete} univariate distributions. In this case the basis of orthonormal polynomials is the
set of all parity functions, see Section \ref{sec:fourier}.   However, our method  encompasses scenarios where $\mu_i$ can be continuous (e.g.,
Gaussian).   
}

\input{dts}

\input{comparison}

\input{experiment}

\input{acknowledgement}

\bibliographystyle{alpha}
\bibliography{main}
\appendix

\input{appendix_experiment}

\end{document}

%% file: header.tex
\usepackage{amsmath,amsfonts,amssymb}
\usepackage{mathtools}
\usepackage{amsthm} 
\usepackage{latexsym}
\usepackage{relsize}

\usepackage[capitalise]{cleveref}
\usepackage{xcolor}
\usepackage{dsfont}
\usepackage{longtable}
\usepackage{caption}
\usepackage{footnote}
\makesavenoteenv{tabular}
\makesavenoteenv{table}
\usepackage[colorlinks,linkcolor=blue,citecolor=blue]{hyperref}

\newcommand{\SHOR}{Harmonica}

\newcommand{\Var}{\mathbf{Var}}
\newcommand{\E}{\mathbb{E}}
\newcommand{\R}{\mathbb{R}}
\newcommand{\N}{\mathbb{N}}
\newcommand{\X}{\mathcal{X}}
\newcommand{\D}{\mathcal{D}}

\def\fhat{\hat{f}}

\newcommand{\ignore}[1]{}

\newcommand{\eh}[1]{\noindent{\textcolor{magenta}{\{{\bf EH:} \em #1\}}}}

\theoremstyle{plain}
\newtheorem{theorem}{Theorem}
\newtheorem{lemma}[theorem]{Lemma}
\newtheorem{corollary}[theorem]{Corollary}

\newtheorem{fact}[theorem]{Fact}

\newtheorem*{theorem*}{Theorem}
\newtheorem*{lemma*}{Lemma}
\newtheorem*{corollary*}{Corollary}
\newtheorem*{proposition*}{Proposition}
\newtheorem*{claim*}{Claim}
\newtheorem*{fact*}{Fact}
\newtheorem*{observation*}{Observation}

\theoremstyle{definition}
\newtheorem{definition}[theorem]{Definition}

\newtheorem*{definition*}{Definition}
\newtheorem*{remark*}{Remark}
\newtheorem*{example*}{Example}

 \theoremstyle{plain}
\newtheorem*{theoremaux}{\theoremauxref}
\gdef\theoremauxref{1}

\DeclareMathAlphabet{\mathbfsf}{\encodingdefault}{\sfdefault}{bx}{n}

\DeclareMathOperator*{\argmin}{arg\,min}

\let\Pr\relax
\DeclareMathOperator{\Pr}{\mathbb{P}}

\newcommand{\mycases}[4]{{
\left\{
\begin{array}{ll}
    {#1} & {\;\text{#2}} \\[1ex]
    {#3} & {\;\text{#4}}
\end{array}
\right. }}

\newcommand{\reals}{\mathbb{R}}
\newcommand{\eps}{\varepsilon}

\renewcommand{\leq}{~\le~}
\renewcommand{\geq}{~\ge~}

\let\oldtfrac\tfrac
\renewcommand{\tfrac}[2]{\smash{\oldtfrac{#1}{#2}}}

\let\nablaold\nabla
\renewcommand{\nabla}{\nablaold\mkern-2.5mu}

%% file: dts.tex

In the rest of this section we proceed to prove the main lemma and derive the theorem.   Recall the Chebyshev inequality:

\begin{fact}[Multidimensional Chebyshev inequality] \label{fact:cheb} 
	Let $X$ be an $m$ dimensional random vector, with expected value $\mu=\E[X]$, and covariance matrix $V=\E[(X-\mu)(X-\mu)^T]$. 
	
	If $V$ is a positive definite matrix, for any real number
	$\delta > 0$:
	\[
	\Pr(\sqrt{(X-\mu)^T V^{-1}(X-\mu)}>\delta)\leq \frac{m}{\delta^2}
	\]
\end{fact}

\begin{proof}[Proof of Lemma \ref{thm:main}]
For ease of notation we assume $K = 1$.
Let $f$ be an $(\epsilon/4, s,
d)$-bounded function written in the orthonormal basis as  $\sum_{S} \hat{f}(S) \psi_{S}$.  We can equivalently write $f$ as 	$f = h + g$, where $h$ 
is a degree $d$ polynomial that  only includes coefficients of
magnitude at least $\eps/4s$ and the constant term of the polynomial expansion of $f$.
	
Since $ L_1(f) =  \sum_S |\hat{f}_S |  \leq s$, by Fact
\ref{sparsefact} we have that $h$ is $4s^2/\eps + 1$ sparse. The function $g$ is thus the sum of the remaining $\hat{f}(S)\psi_{S}$ terms not included in $h$.

Draw $m$ (to be chosen later) random labeled examples
$\{(z^1,y^1),\ldots,(z^m,y^m)\}$ and enumerate all $N = n^d$ basis
functions $\psi_{S}$ for $|S| \leq d$ as $\{\psi_{1},\ldots,\psi_N\}$.
Form matrix $A$ such that $A_{ij} = \psi_{j}(z^i)$ and consider the
problem of recovering $4s^2/\eps + 1$ sparse $x$ given $Ax + e  =
y $ where $x$ is the vector of coefficients of $h$, the $i$th
entry of $y$ equals $y^i$, and $e_{i} = g(z^i)$.

We will prove that with constant probability over the choice $m$
random examples, $\|e\|_2 \leq \sqrt{\epsilon m}$.
Applying Theorem \ref{GG} by setting $\eta=\sqrt{\epsilon}$ and observing that
$\sigma_{4s^2/\epsilon + 1}(x)_1 = 0$, we will recover $x'$ such that $\|x -
x'\|_2^{2} \leq c_2^2 \epsilon$ for some constant $c_2$.   As such, for the function $\tilde{f} =
\sum_{i=1}^{N} x'_{i} \psi_i$ we will have $\E[\|h - \tilde{f}\|^2] \leq
c_2^2 \epsilon$ by Parseval's identity.  Note, however, that we may rescale $\epsilon$ by
constant factor $1/(2c_2^2)$ to obtain error $\epsilon/2$ and only incur an additional
constant (multiplicative) factor in the sample complexity bound. 

By the definition of $g$, we have 
	
\begin{equation}
 \|g\|^2	= \left ( 	\sum_{S, |S| > d} \hat{f}(S)^2 + \sum_{R} \hat{f}(R)^2\right ) 
 \label{eqn:g_var}
\end{equation}

where each $\hat{f}(R)$ is of magnitude at most $\eps/4s$.   By Fact
\ref{sparsefact} and Parseval's identity we have $\sum_{R}
\hat{f}(R)^2 \leq \epsilon/4$.  Since $f$ is $(\epsilon/4, d)$-concentrated we have $\sum_{S, |S| > d} \hat{f}(S)^2 \leq
\epsilon/4$.  Thus,  $\|g\|^2$ is at most $\epsilon/2$. 
Therefore, by triangle inequality $
\E[\|f - \tilde{f}\|^2]\leq 
\E[\|h - \tilde{f}\|^2]
+\E[\|g\|^2]\leq \epsilon$.
	
It remains to bound $\|e \|_2$.  Note that since the examples are
chosen independently, the entries $e_i = g(z^i)$ are independent
random variables.  Since $g$ is a linear combination of orthonormal
monomials (not including the constant term), we have $\E_{z \sim
D}[g(z)] = 0$.  Here we can apply linearity of variance (the covariance of
$\psi_{i}$ and $\psi_{j}$ is zero for all $i \neq j$) and calculate the variance
	
\[ \Var(g(z^i))	= ( 	\sum_{S, |S| > d} \hat{f}(S)^2 + \sum_{R} \hat{f}(R)^2) \]

With the same calculation as (\ref{eqn:g_var}), we know 
 $\Var(g(z^i) )$ is at most $\epsilon/2$. 
	
Now consider the covariance matrix $V$ of the vector $e$
	which equals $\E[ e e^\top ] $ (recall every entry
	of $e $ has mean $0$).  Then $V$ is a diagonal matrix (covariance between two independent samples is zero), and every diagonal entry is at most $\eps/2$. 	
	Applying Fact \ref{fact:cheb} we have 
	\[
	\Pr(\|e \|_2 >\sqrt{\frac{\epsilon}{2}} \delta)\leq \frac{m}{\delta^2}.
	\]
	
	Setting $\delta = \sqrt{2 m}$, we conclude that 
	$\Pr(\|e \|_2 > \sqrt{\eps m}) \leq \frac12$.  Hence with probability at least $1/2$, we have that $\|e \|_2 \leq
	\sqrt{\epsilon m}$. From Theorem \ref{GG}, we may choose $m =
	\tilde{O}(s^2/\epsilon \cdot  \log n^d)$.  This completes the proof.
        Note that the probability $1/2$ above can be boosted to any
        constant probability with a constant factor loss in sample complexity. \qedhere

\end{proof}

\noindent{\bf Remark:} Note that the above proof also holds in the
{\em adversarial} or {\em agnostic} noise setting. That is, an adversary could add a
noise vector $v$ to the labels received by the learner.  In this case,
the learner will see label vector $y = Ax + e + v$.  If
$\|v\|_2 \leq \sqrt{\gamma m}$, then we will recover a polynomial with
squared-error at most $\eps + O(\gamma)$ via re-scaling $\epsilon$ by a
constant factor and applying the triangle inequality to $\|e +
v\|_2$. \\

While this noisy recovery lemma is the basis for our enhanced
algorithm in the next section as well as the learning-theoretic result
on learning of decision trees detailed in the next subsection, it does
not imply recovery of the global optimum. The reason is that noisy
recovery guarantees that we output a hypothesis {\em close} to the
underlying function, but even a single noisy point can completely
change the optimum.

Nevertheless, we can use our techniques to prove recovery of optimality for functions that are computed {\it exactly} by a sparse, low-degree polynomial.

\begin{proof}[Proof of Theorem \ref{thm:main-harmonica1}]
	There are at most $N = n^{d}$ polynomials $\psi_{S}$ with $|S| \leq d$.
	Let the enumeration of these polynomials be $\psi_{1},\ldots,
	\psi_{N}$.  Draw $m$ labeled examples $\{(z^1,y^1), \ldots,(z^m,y^m)\}$
	independently from ${\cal D}$ and
	construct an $m \times N$ matrix $A$ with $A_{ij} =
	\psi_{j}(z^i)$.   Since $f$ can be written as an $s$ sparse linear
	combination of $\psi_{1},\ldots,\psi_{N}$, there exists an
	$s$-sparse vector $x$ such that $Ax = y$ where the $i$th entry of
	$y$ is $y^i$.  Hence we can apply Theorem \ref{GG} to recover $x$
	exactly. These are the $s$ non-zero coefficients of
	$f$'s expansion in terms of $\{\psi_{S}\}$.  Since $f$ is recovered exactly, its minimizer is found in the optimization step.
\end{proof}


\subsection{Application: Learning Decision Trees in Quasi-polynomial Time
and Polynomial Sample Complexity} \label{sec:dt}

Here we observe that our results imply new bounds for decision-tree
learning.  
For example, we obtain the first quasi-polynomial
time algorithm for learning decision trees with respect to the uniform
distribution on $\{-1,1\}^n$ with polynomial sample complexity and an
optimal dependence on the error parameter $\epsilon$:

\begin{corollary}
	Let ${\cal X} = \{-1,1\}^n$ and let ${\cal C}$ be the class of all
	decision trees of size $s$ on $n$ variables.  Then ${\cal C}$ is
	learnable with respect to the uniform distribution in time $n^{O(\log
		(s/\epsilon))}$ and sample complexity $m = \tilde{O}(s^{2}
              /\epsilon \cdot \log n)$.   Further, if the labels are
              corrupted by arbitrary noise vector $v$ such that
              $\|v\|_2 \leq \sqrt{\gamma m}$, then the output
              classifier will have squared-error at most $\epsilon +
              O(\gamma)$. 
\end{corollary}

\vspace{-0.6em}
\begin{proof}
	As mentioned earlier, the orthonormal polynomial basis for the class
	of Boolean functions with respect to the uniform distribution on
	$\{-1,1\}^n$ is the class of parity functions $\{\chi_{S}\}$
	for $S \subseteq \{-1,1\}^n$.  Further, it is easy to show that
        for Boolean function $f$, if
        $\E[(h - f)^2] \leq \epsilon$ then $\Pr[\mathsf{sign}(h(x))
        \neq f(x)] \leq \epsilon$.   The corollary now follows by applying Lemma
	\ref{thm:main} and two known structural facts about decision trees:
	1) a tree of size $s$ is
	$(\eps, \log(s/\eps))$-concentrated and has $L_1$ norm bounded by
	$s$ (see e.g., Mansour \cite{MansourSurvey}) and 2) by Fact
        \ref{sparsefact}, for any function $f$ with $L_1$ norm bounded by $s$
        (i.e., a decision tree of size $s$), there exists an
        $s^2/\eps$ sparse function $g$ such that $\E[(f-g)^2] \leq
        \epsilon$.  The noise tolerance property follows immediately from the
        remark after the proof of Lemma \ref{thm:main}.
\end{proof}

\noindent{\bf Comparison with the ``Low-Degree'' Algorithm}.  Prior
work for learning decision trees (more generally Boolean functions
that are approximated by low-degree polynomials) used the celebrated
``low-degree'' algorithm of Linial, Mansour, and Nisan \cite{LMN}.
Their algorithm uses random sampling to estimate each low-degree
Fourier coefficient to high accuracy.  In contrast, the
compressed-sensing approach of Stobbe and Krause \cite{StobbeKrause}
takes advantage of the incoherence of the design matrix and gives
results that seem unattainable from the ``low-degree'' algorithm.

For learning noiseless, Boolean decision trees, the low-degree
algorithm uses quasipolynomial time and sample complexity
$\tilde{O}(s^2/\eps^2 \cdot \log n)$ to learn to accuracy $\epsilon$.
It is not
clear, however, how to obtain any noise tolerance from their approach.

For general real-valued decision trees where $B$ is an upper bound on
the maximum value at any leaf of a size $s$ tree, our algorithm will
succeed with sample complexity $\tilde{O}(B^2s^2/\epsilon \cdot \log
n)$ and be tolerant to noise while the low-degree algorithm will use
$\tilde{O}(B^4s^2/\epsilon^2 \cdot \log n)$ (and will have no
noise tolerance properties).  Note the improvement in the dependence
on $\epsilon$ (even in the noiseless setting), which is a consequence
of the RIP property of the random orthonormal family.





%% file: comparison.tex
\section{Harmonica: The Full Algorithm}

Rather than applying Algorithm \ref{alg:basic-harmonica} directly, we found that performance is greatly enhanced by iteratively using Procedure \ref{alg:2} to estimate the most influential hyperparameters and their optimal values. 

In the rest of this section we describe this iterative heuristic,
which essentially runs Algorithm \ref{alg:basic-harmonica} for
multiple stages.  More concretely, we continue to invoke the PSR subroutine until the
search space becomes small enough for us to use a
``base'' hyperparameter optimizer (in our case either SH or
Random Search).  

The space of minimizing assignments to a multivariate polynomial is a
highly non-convex set that may contain many distinct points.  As such,
we take an average of several of the best minimizers (of subsets of
hyperparameters) during each stage. 

In order to describe this formally we need the following definition
of a restriction of function:

\begin{definition}[restriction  \cite{booleananalysis}]
	Let $f\in \{-1,1\}^{n}\mapsto \R$, 
$J\subseteq [n]$, and $z\in \{-1,1\}^{ J}$ be given.
	We call $(J,z)$ a restriction pair of function $f$.  We denote $f_{J,z}$  the function over $n-|J|$ variables given by setting the variables of $J$ to $z$. 
\end{definition}

We can now describe our main algorithm (Algorithm
\ref{alg:alg1}). Here $q$ is the number of stages for which we apply
the PSR subroutine, and the restriction size $t$ serves as a
tie-breaking rule for the best minimizers (which can be set to $1$).

\begin{algorithm}[h!]
\caption{\SHOR-$q$}
\label{alg:alg1}
\begin{algorithmic}[1]
\STATE Input: oracle for $f$, number of samples $T$, sparsity $s$,
degree $d$, regularization parameter $\lambda$, number of stages $q$,
restriction size $t$, base hyperparameter optimizer ALG.
\FOR {stage $i=1$ to $q$}
\STATE  Invoke PSR$(f,T,s,d,\lambda)$ (Procedure \ref{alg:2}) to obtain
$(g_i,J_i)$, where $g_i$ is a function defined on variables specified by index
set $J_i \subseteq [n]$.  
\STATE Let $M_i = \{ x^\star_1,...,x^\star_t \} =  \arg \min g_i(x)$
be the best $t$ minimizers of $g_i$.  
\STATE Let $f_i = \E_{k \in [t]} [ f_{J_i,x^\star_k} ]  $ be the
expected restriction of $f$ according to minimizers
$M_i$.\footnotemark 
\STATE  Set $f = f_i$.


\ENDFOR

\RETURN Search for the global minimizer of $f_q$ using base optimizer ALG 
\end{algorithmic}
\end{algorithm}

\subsection{Algorithm attributes and heuristics}
\label{sec:appendix:alg_compare}

\paragraph{Scalability.}
If the hidden function
if $s$-sparse, \SHOR{} can find such a sparse function using
$\tilde {O}(s \log s)$ samples. If 
at every stage of \SHOR{}, the target function can be approximated by
an $s$ sparse function, we only need $\tilde {O}(qs \log s)$ samples
where $q$ is the number of stages. For real world applications such as
deep neural network hyperparameter tuning, it seems (empirically)
reasonable to assume that the hidden function is indeed sparse at
every stage (see Section \ref{sec:exp:resnet}).

For Hyperband \cite{hyperband}, SH \cite{successive} or Random Search, even if the function is  $s$-sparse, 
in order to cover the optimal configuration by random sampling, we
need $\Omega(2^{s})$ samples. 

\footnotetext{In order to evaluate $f_i$, 
	we first sample $k\in [t]$ to obtain $f_{J_i,x_k^*}$, and then 
	evaluate $f_{J_{i},x_k^*}$.
}

\paragraph{Optimization time.}
\SHOR{} runs the Lasso \cite{lasso} algorithm after each stage to solve (\ref{eqn:shalom}), which is a well studied convex optimization problem and has very fast implementations.
Hyperband and SH are also efficient in terms of running time as a function of the number of function evaluations, and require sorting or other simple computations. 
The running time of Bayesian optimization 
is cubic in number of function evaluations, which limits applicability
for large number of evaluations / high dimensionality, as we shall see in Section \ref{sec:synthetic_func}.

\paragraph{Parallelizability.}
\SHOR{}, similar to Hyperband, SH, and Random Search, has straightforward parallel implementations. In every stage of those algorithms, we could simply evaluate the objective functions over randomly chosen points in parallel. 

It is hard to run Bayesian optimization algorithm in parallel due to its inherent serial nature.
Previous works explored variants in which multiple points are evaluated at the same time in parallel \cite{parallelbo}, 
though speed ups do not grow linearly in the number of machines, and the batch size is usually limited to a small number.

\paragraph{Feature Extraction.}
\SHOR{}  is able to extract important features with weights in each stages, which automatically sorts all the features according to their importance. 
See Section \ref{sec:appendix:important_feature_table}.

%% file: experiment.tex
\section{Experiments with training deep networks} 
\label{sec:exp:resnet}
We compare \SHOR{}\footnote{A python implementation of \SHOR{} can be found at \url{
		https://github.com/callowbird/Harmonica
}} with Spearmint\footnote{\url{https://github.com/HIPS/Spearmint.git}} \cite{bayesianOPT},  Hyperband, SH\footnote{We implemented a parallel version of Hyperband and SH in Lua.}  and Random Search. 
Both Spearmint and Hyperband are state-of-the-art algorithms, and it is observed that Random Search 2x
(Random Search with doubled function evaluation resources)
is a very competitive benchmark that beats many algorithms\footnote{E.g., see \cite{recht1,recht2}. }.

Our first experiment is over training residual network on Cifar-10 dataset\footnote{\url{https://github.com/facebook/fb.resnet.torch}}.
We included $39$ binary hyperparameters,  including 
initialization, optimization method, learning rate schedule, momentum
rate, etc.  Table \ref{tab:options} (Section
\ref{sec:appendix:options}) details the hyperparameters considered. We
also include $21$ dummy variables to make the task more
challenging. Notice that Hyperband, SH, and Random Search 
are agnostic
to the dummy variables in the sense that they just set the value of dummy variables randomly, therefore select essentially the same set of configurations with or without the dummy variables.  Only Harmonica and Spearmint are sensitive to the dummy variables as they try to learn the high dimensional function space. To make a fair comparison, we run Spearmint without any dummy variables.

As most hyperparameters have a consistent effect as the network becomes deeper, 
a common hand-tuning strategy  is 
``tune on small network, then apply the knowledge to big network'' (See discussion in Section \ref{sec:tune_small}).
\SHOR{} can also exploit this strategy as it selects important features stage-by-stage. More specifically, 
during the feature selection stages, 
we run \SHOR{} for tuning an $8$ layer neural network with $30$ training epochs. 
At each stage, we take
$300$ samples to extract $5$ important features, and set restriction size $t=4$ (see Procedure \ref{alg:2}).
 After that, we fix all the important
 features, and run the
SH or Random Search as our base algorithm on the big $56$ layer neural network for training the whole $160$ epochs\footnote{Other algorithms like Spearmint, Hyperband, etc. can be used as the base algorithms as well. }. To clarify, ``stage'' means the stages of the hyperparameter algorithms, while ``epoch'' means the epochs for training the neural network. 

\subsection{Performance}

	\begin{figure}
	\centering
	\includegraphics[width=\textwidth]{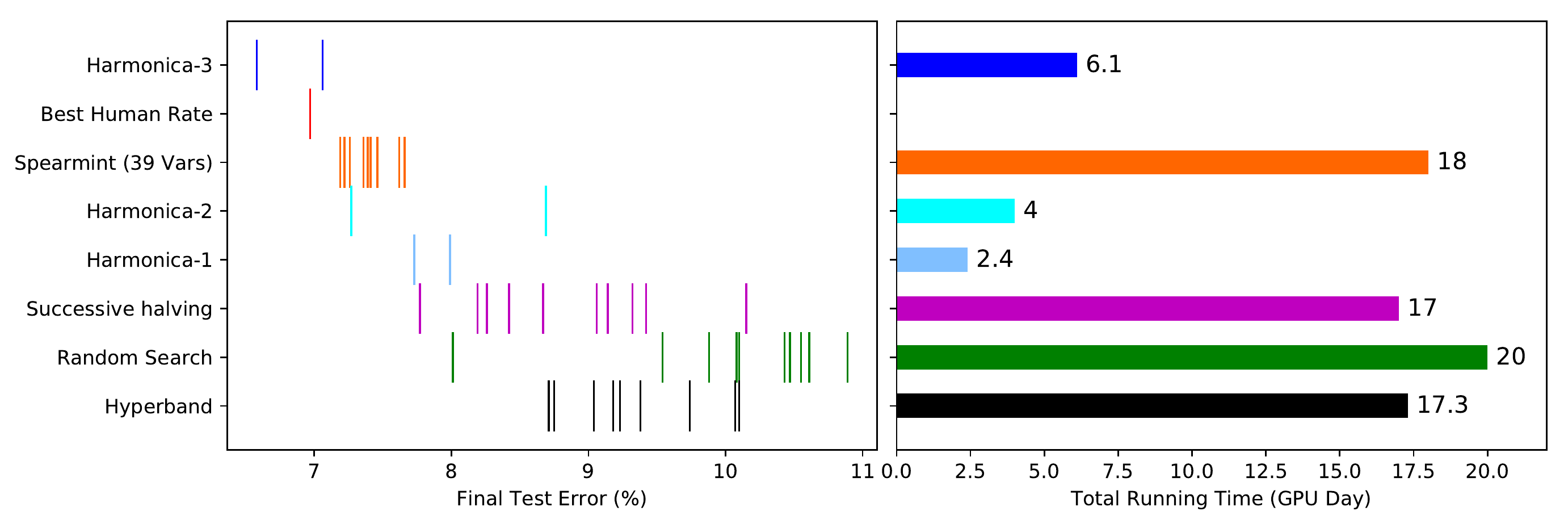}
	\caption{{\small Distribution of the best results and running time of different algorithms}}
	\label{fig:best}	
	\end{figure}

We tried three versions of \SHOR{} for this experiment, \SHOR{} with 
1 stage (\SHOR{}-1), 2 stages (\SHOR{}-2)
and 3 stages (\SHOR{}-3). All of them use 
SH as the base algorithm. 
The top $10$ test error results and running times of the different  algorithms are depicted in Figure \ref{fig:best}. SH based algorithms may return fewer than 10 results.
For more runs of variants of Harmonica and its resulting test error, see Figure \ref{fig:hyper_best}	 (the results are similar to Figure \ref{fig:best}).


\paragraph{Test error and scalability:} \SHOR{}-1 uses less than 
$1/7$ time of Hyperband and $1/8$ time compared with Random Search,
but gets better results than the competing algorithms. It beats the Random Search $8$x benchmark (stronger than Random Search 2x benchmark of \cite{hyperband}). \SHOR{}-2 uses slightly more time, but is able to find better results, which are comparable with Spearmint with $4.5\times$ running time. 

\paragraph{Improving upon human-tuned parameters:} 
\SHOR{}-3 obtains a better test error ($6.58\%$) as compared to the best hand-tuning rate $6.97\%$ reported in \cite{resnet}\footnote{
$6.97\%$ is the rate obtained by residual network, 
and there are new network structures like wide residual network \cite{wideresnet} or densenet \cite{densenet} that achieve better rates for Cifar-10.}. \SHOR{}-3 uses only 6.1 GPU days, which is less than half day in our environment, as we have 20 GPUs running in parallel. Notice that we did not cherry pick the results for \SHOR{}-3. In Section \ref{sec:hyperparameter_for_harmonica} we show that by running \SHOR{}-3 for longer time, one can obtain a few other solutions better than hand tuning. 

\paragraph{Performance of provable methods:} 
\SHOR{}-1 has noiseless and noisy recovery guarantees (Lemma \ref{thm:main}), which are validated experimentally. 

\vspace{10pt}



%

\subsection{Average Test Error For Each Stage}
We computed the average test error among $300$ random samples for an $8$ layer network with $30$ epochs after each stage. See Figure \ref{fig:drop}. After selecting $5$ features in stage $1$, the average test error drops from $60.16$ to $33.3$, which indicates the top $5$ features are very important. As we proceed to stage $3$, the improvement on test error becomes less significant as the selected features at stage $3$ have mild contributions.

\subsection{Hyperparameters for Harmonica}
\label{sec:hyperparameter_for_harmonica}
To be clear, \SHOR{} itself has six hyperparameters that one needs to set including
the number of stages, $\ell_1$ regularizer for Lasso,
the number of features selected per stage, base algorithm, small
network configuration, and the number of samples per
stage. Note, however, that we have reduced the search space of general
hyperparameter optimization down to a set of only six
hyperparameters.  Empirically, our algorithm is robust to different
settings of these parameters, and we did not even attempt to tune
some of them (e.g., small network configuration).

	\begin{figure}
	\centering
	\includegraphics[width=\textwidth]{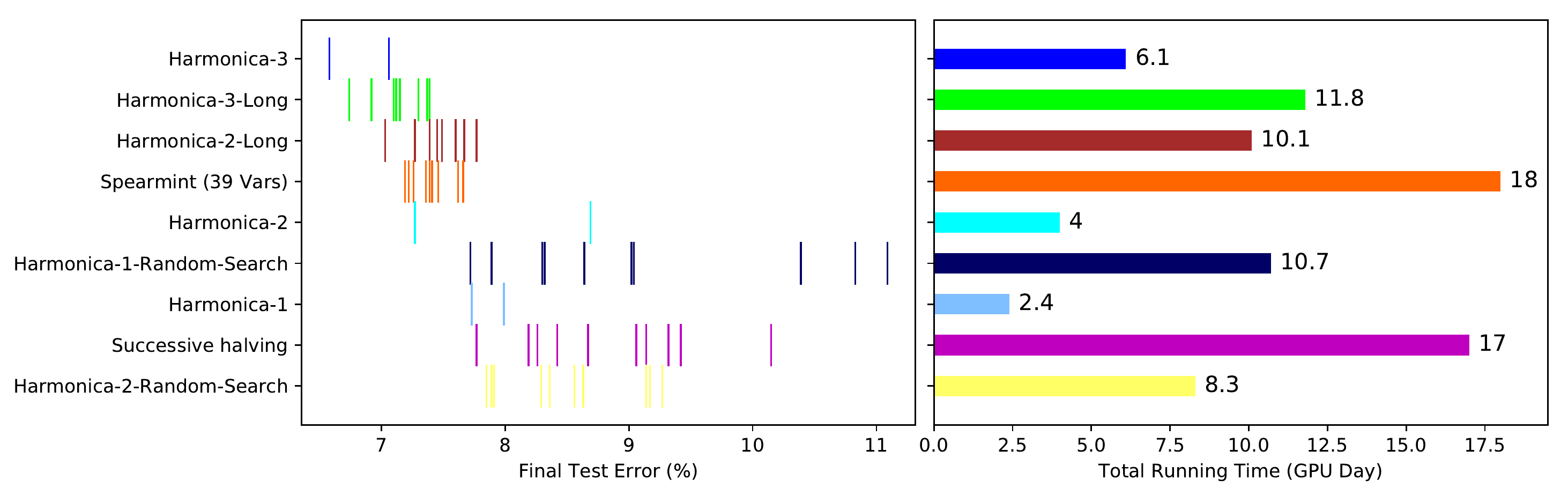}
	\caption{{\small Comparing different variants of \SHOR{} with SH on test error and running time}}
	\label{fig:hyper_best}	
\end{figure}

\paragraph{Base algorithm and \#stages.}
We tried different versions of \SHOR{}, including \SHOR{} with 1 stage, 2 stages and 3 stages using SH as the base algorithm (\SHOR{}-1, \SHOR{}-2, \SHOR{}-3), with 1 stage and 2 stages using Random Search as the base algorithm (\SHOR{}-1-Random-Search, \SHOR{}-2-Random-Search), and with 2 stages and 3 stages running SH as the base for longer time (\SHOR{}-2-Long, \SHOR{}-3-Long).
As can be seen in Figure \ref{fig:hyper_best}, most variants produce better results than SH and use less running time.
Moreover, if we run \SHOR{} for longer time, we will obtain more stable solutions with less variance in test error.



\label{sec:appendix:sensitivity}

\begin{table} 
	\captionof{table}{{\small Stable ranges for parameters in Lasso}	}
	\label{tab:stable}
	\centering
	{\small
		\begin{tabular}{|c|c|c|c|} 
			\hline
			Parameter& Stage 1& Stage 2 & Stage 3 \\
			\hline 
			$\lambda$ & $[0.01,4.5]$& $[0.1,2.5]$ & $[0.5,1.1]$\\
			\hline 
			\#Samples & $\geq 250$& $\geq 180$ & $\geq 150$\\		
			\hline  
		\end{tabular}
	}
\end{table}

\paragraph{Lasso parameters are stable.} See Table \ref{tab:stable} for stable range for regularization term  $\lambda$ and  the number of samples. Here stable range means as long as the parameters are set in this range, the top $5$ features and the signs of their weights (which are what we need for computing $g(x)$ in Procedure \ref{alg:2}) do not change. In other words, the feature selection outcome is not affected. 
When parameters are outside the stable ranges, usually the top features are still unchanged, and
we miss only one or two out of the five features.

\paragraph{On the degree of features.}
We set degree to be three because it does not find any important
features with degree larger than this. Since Lasso can be solved
efficiently (less than $5$ minutes in our experiments), the choice of degree can be decided automatically. 


%



\subsection{Experiments with Synthetic functions}
\label{sec:synthetic_func}
            
\begin{figure}
		\begin{minipage}{0.45\textwidth}
		\centering
		\includegraphics[width=170pt]{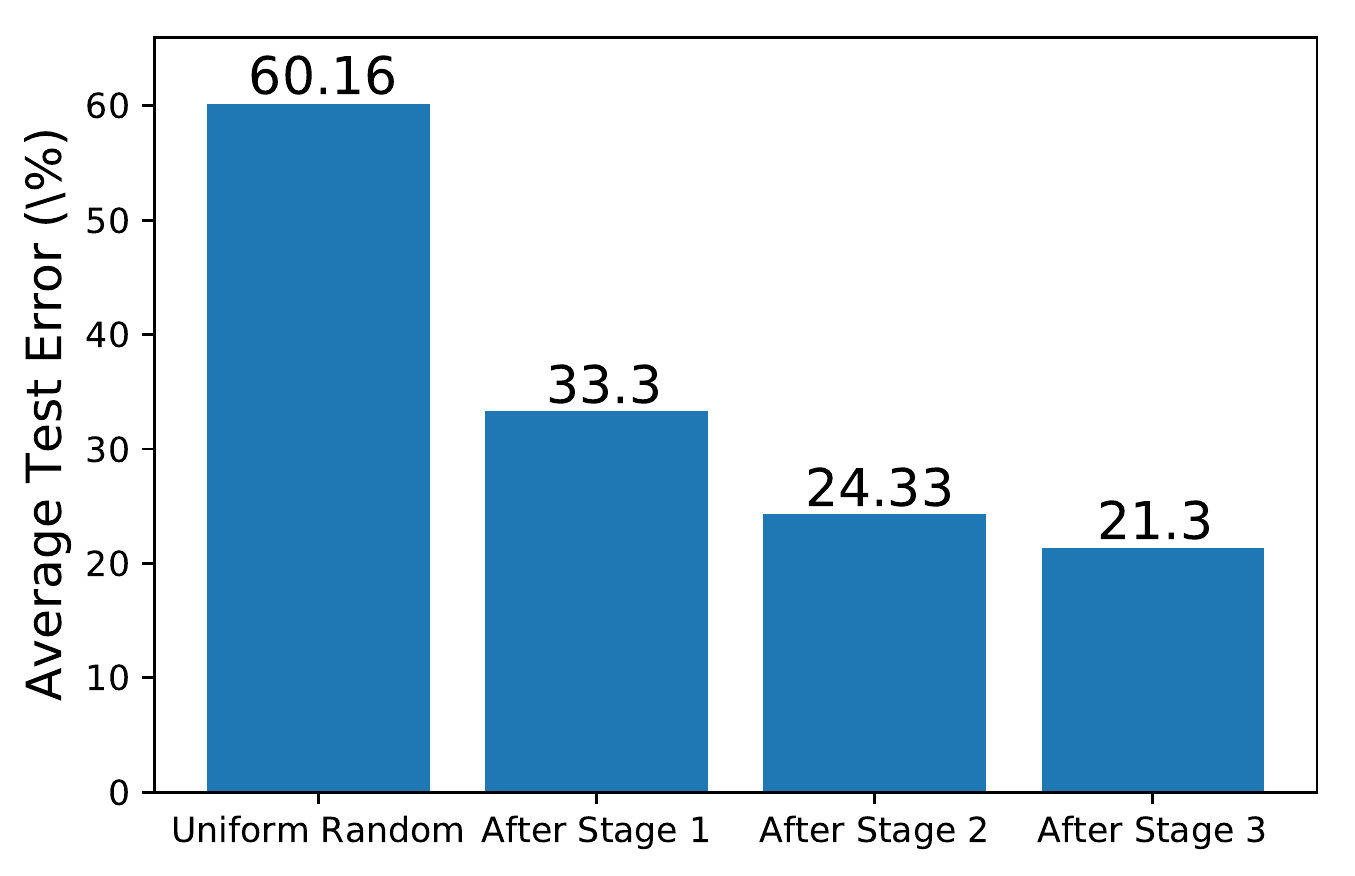}
		\captionof{figure}{{\small Average test error drops.}}
		\label{fig:drop}	
	\end{minipage}	
	\begin{minipage}{0.45\textwidth}
	\centering
	\includegraphics[width=180pt]{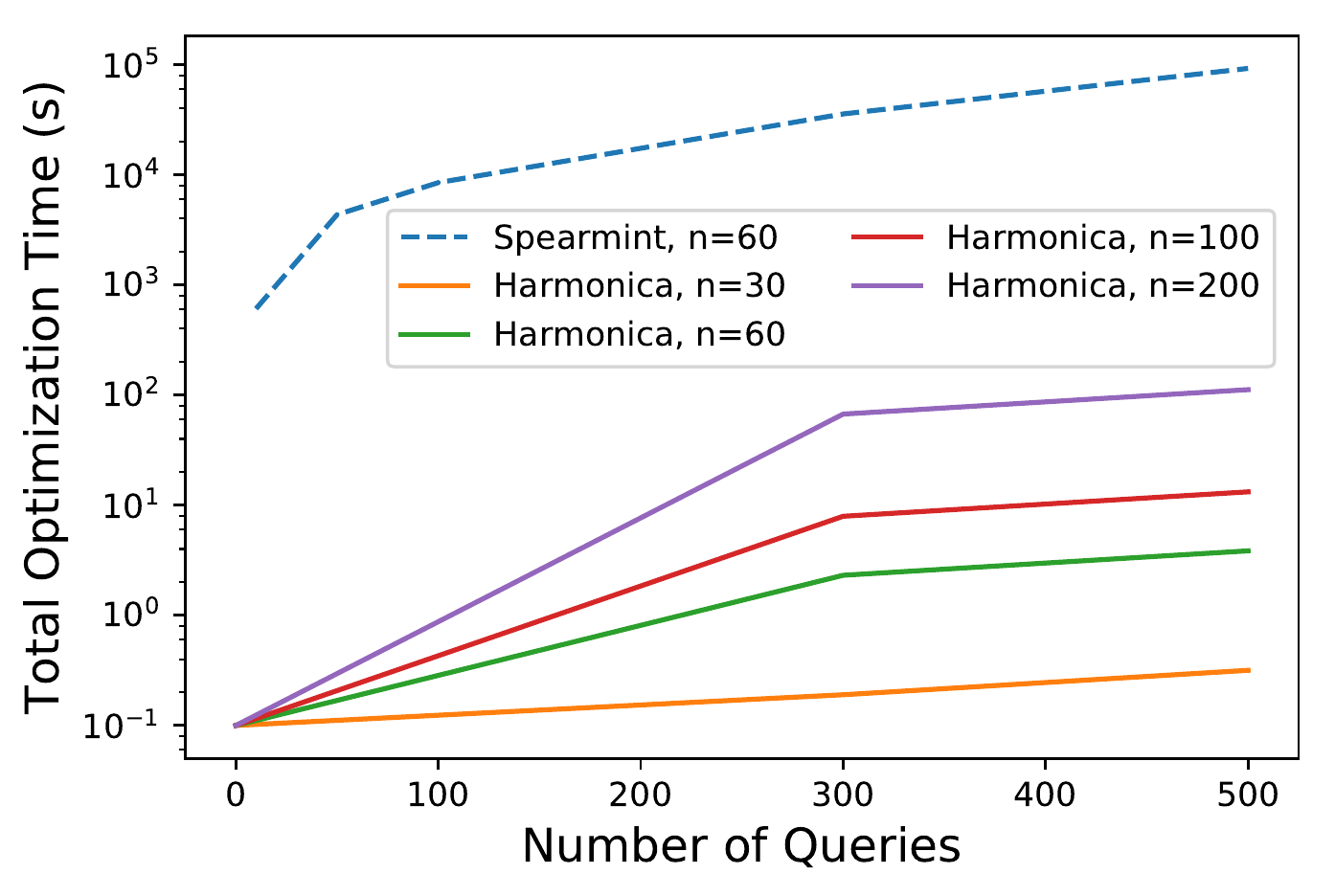}
	\caption{{\small Optimization time comparison}}
	\label{fig:runTime}	
\end{minipage}

\end{figure}
Our second experiment considers a synthetic hierarchically bounded function $h(x)$. We run \SHOR{} with $100$ samples, $5$ features selected per stage, for $3$ stages, using degree $3$ features. 
See Figure \ref{fig:runTime} for optimization time comparison. We only plot the optimization time for Spearmint when $n=60$, which takes more than one day for $500$ samples. \SHOR{} is several magnitudes faster than Spearmint.
In Figure \ref{fig:noise}, we show that \SHOR{} is able to estimate the hidden function with error proportional to the noise level. 

The synthetic function $h(x)\in  \{-1,+1\}^{n}\rightarrow \R$ is defined as follows.
$h(x)$ has three stages, and in $i$-th stage ($i=0,1,2$), it has $32^i$ sparse vectors $s_{i,j}$ for $j=0,\cdots, 32^i-1$. 
Each $s_{i,j}$ contains $5$ pairs of weight $w_{i,j}^k$ and feature $f_{i,j}^k$ for $k=1,\cdots 5$, where 
$w_{i,j}^k \in [10+10^{-i}, 10+10^{2-i}]$.
and $f_{i,j}^k$ is a monomial on $x$ with degree at most $3$. Therefore, for input $x\in \mathbb{R}^{n}$, the sparse vector $s_{i,j}(x)=\sum_{k=1}^5 w_{i,j}^k f_{i,j}^k(x)$. Since $x\in \{
-1,+1
\}^n$, $f_{i,j}^k(x)$ is binary. Therefore, $\{f_{i,j}^k(x)\}_{k=1}^5$ contains $5$ binaries and represents a integer in $[0,31]$, denoted as $c_{i,j}(x)$. Let $h(x)=s_{1,1}(x)+ s_{2,c_{1,1}(x)}(x)+
s_{3,c_{1,1}(x)*32+c_{2,c_{1,1}(x)}(x)}(x)+\xi$, where $\xi$ is the noise uniformly sampled from $[-A,A]$ ($A$ is the noise level). 
In other words, in every stage $i$ we will get a sparse vector $s_{i,j}$. Based on $s_{i,j}(x)$, we pick a the next sparse function and proceed to the next stage.

\begin{figure}[h!]
	\centering
	\includegraphics[width=160pt]{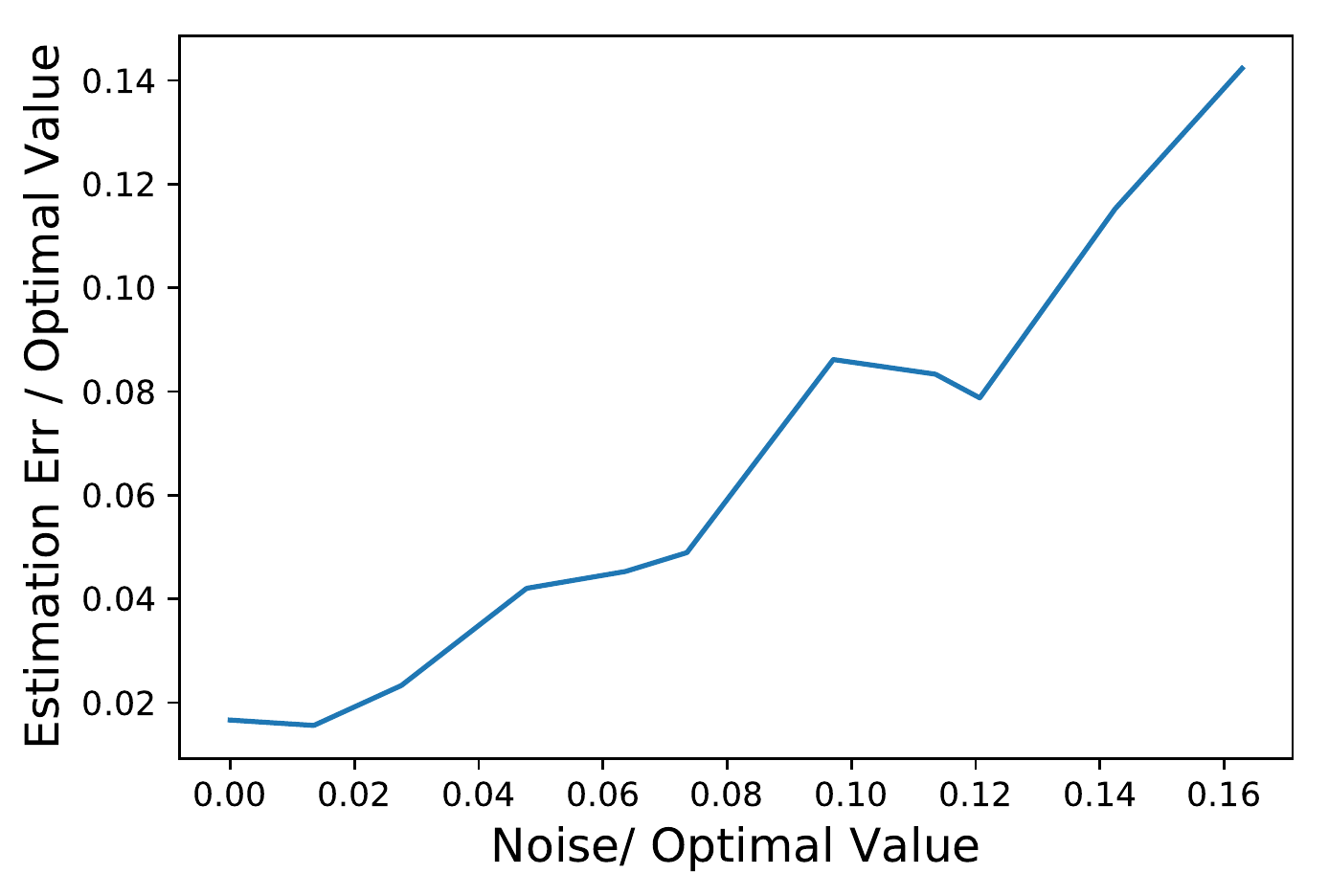}
	\caption{The estimation error of \SHOR{} is linear in noise level.}
	\label{fig:noise}	
\end{figure}

%% file: acknowledgement.tex
\section{Acknowledgements}

Thanks to Vitaly Feldman for pointing out the work of Stobbe and Krause \cite{StobbeKrause}.  We thank Sanjeev Arora for helpful discussions and encouragement.  Elad Hazan is supported by NSF grant 1523815. This project is supported by a Microsoft Azure research award and Amazon AWS research award.

%% file: appendix_experiment.tex

\newpage

\section{Experimental details}

\subsection{Options}
\label{sec:appendix:options}
\begin{longtable}{|p{6cm}|p{10cm}|}
		\caption{60 options used in Section \ref{sec:exp:resnet}}
	\label{tab:options}\\	
		\hline
Option Name & Description\\
		\hline
01. Weight initialization & Use standard initializations or other  initializations?\\
		\hline
02. Weight initialization (Detail 1) & 
Xavier Glorot \cite{XavierGlorot}, Kaiming \cite{kaiming_init}, $1/n$, or $1/n^2$?
\\
		\hline
03. Optimization method  & SGD or ADAM? \cite{ADAM}\\
		\hline
04. Initial learning rate  & $\geq0.01$ or $<0.01$?\\
		\hline
05. Initial learning rate (Detail 1) & $\geq 0.1$, $<0.1$, $\geq0.001$, or $<0.001$?\\
		\hline
06. Initial learning rate (Detail 2)& 
0.3, 0.1, 0.03, 0.01, 0.003, 0.001, 0.0003, or  0.0001?
\\
		\hline
07. Learning rate drop   & Do we need to decrease learning rate as we train? Yes or No?\\
		\hline
08. Learning rate first drop time   & If drop learning rate, when is the first time to drop by $1/10$? Epoch 40 or Epoch 60?\\
		\hline
09. Learning rate second drop time  & If drop learning rate, when is the second time to drop by $1/100$? Epoch 80 or Epoch 100?\\
		\hline
10. Use momentum \cite{momentum}  & Yes or No?\\
		\hline
11. Momentum rate & If use momentum, rate is 0.9 or 0.99?\\
		\hline
12. Initial residual link weight  & What is the initial residual link weight? All constant 1 or a random number in $[0,1]$?\\
\hline
13. Tune residual link weight  & Do we want to use back propagation to tune the weight of residual links? Yes or No?\\
		\hline
14. Tune time of residual link weight  & When do we start to tune residual link weight? At the first epoch or epoch 10?\\
	\hline
15. Resblock first activation  & Do we want to add activation layer after the first convolution? Yes or No?\\
	\hline
16. Resblock second activation  & Do we want to add activation layer after the second convolution? Yes or No?\\
	\hline
17. Resblock third activation  & Do we want to add activation layer after adding the residual link? Yes or No?\\
	\hline
18. Convolution bias  &Do we want to have bias term in convolutional layers? Yes or No?\\
	\hline
19. Activation & What kind of activations do we use? ReLU or others?\\
	\hline
20. Activation (Detail 1)  & ReLU, ReLU, Sigmoid, or Tanh?\\
	\hline
21. Use dropout \cite{dropout} & Yes or No?\\
	\hline
22. Dropout rate & If use dropout, rate is high or low?\\
	\hline
23. Dropout rate (Detail 1)  & If use dropout, the rate is 0.3, 0.2, 0.1, or 0.05?\\
	\hline
24. Batch norm \cite{batchnorm} & Do we use batch norm? Yes or No?\\
	\hline
25. Batch norm tuning & If we use batch norm, do we tune the parameters in the batch norm layers? Yes or No?\\
	\hline
26. Resnet shortcut type  & What kind of resnet shortcut type do we use? Identity or others?\\
	\hline
27. Resnet shortcut type (Detail 1) & Identity, Identity, Type B or Type C?\\
	\hline
28. Weight decay & Do we use weight decay during the training? Yes or No?\\
	\hline
29. Weight decay parameter & If use weight decay, what is the parameter? $1e-3$ or $1e-4$?\\
	\hline
30. Batch Size      & What is the batch size we should use? Big or Small?\\
	\hline
31. Batch Size (Detail 1)  & 256, 128, 64, or 32?\\
	\hline
32. Optnet       & An option specific to the code\footnote{https://github.com/facebook/fb.resnet.torch}. Yes or No?\\
	\hline
33. Share gradInput   & An option specific to the code. Yes or No?\\
	\hline
34. Backend  & What kind of backend shall we use? cudnn or cunn? \\
	\hline
35. cudnn running state  & If use cudnn, shall we use fastest of other states?\\
	\hline
36. cudnn running state (Detail 1)  & Fastest, Fastest, default, deterministic\\
	\hline
37. nthreads  & How many threads shall we use? Many or few?\\
	\hline
38. nthreads (Detail 1)  & 8, 4, 2, or 1?\\
	\hline
39-60. Dummy variables & Just dummy variables, no effect at all.\\
\hline	
\end{longtable}

See Table \ref{tab:options} for the specific hyperparameter options that we use in Section \ref{sec:exp:resnet}. For those variables with $k$ options ($k>2$), we use $\log k$ binary variables under the same name to represent them. For example, we have two variables (01, 02) and their binary representation to denote four kinds of possible initializations: Xavier Glorot \cite{XavierGlorot}, Kaiming \cite{kaiming_init}, $1/n$, or $1/n^2$.  

\subsection{Importance features}
\label{sec:appendix:important_feature_table}

We show the selected important features and their weights during the first $3$ stages in Table \ref{tab:important_features}, where each feature is a monomial of variables with degree at most $3$. 
We do not include the 4th stage because in that stage there are no features with nonzero weights.

\textbf{Smart choices on important options}.
Based on Table \ref{tab:important_features}, \SHOR{} will fix the following variables (sorted according to their importance): 
Batch Norm (Yes), Activation (ReLU), Initial learning rate ($[0.001,0.1]$), Optimization method (Adam), 
Use momentum (Yes), 
Resblock first activation (Yes), Resblcok third activation (No), Weight decay (No if initial learning rate is comparatively small and Yes otherwise), Batch norm tuning (Yes). Most of these choices match what people are doing in practice. 

\textbf{A metric for the importance of variables}. The features that \SHOR{} finds can serve as a metric for measuring the importance of different variables. For example, Batch Norm turns out to be the most significant variable, and ReLU is second important. By contrast, Dropout, when Batch Norm is presented, does not have significant contributions. This actually matches with the observations in \cite{batchnorm}.

\textbf{No dummy/irrelevant variables selected}.
Although there are $21/60$ dummy variables, we never select any of them. 
Moreover,  the irrelevant variables like cudnn, backend, nthreads, which do not affect the test error, were not selected. 


\begin{table}[h]
	\centering
	\caption{Important features}
	\label{tab:important_features}
	\begin{tabular}{|c|l|c|}
		\hline
		Stage &	Feature Name	 & Weights \\
		\hline 
		1-1 & 24. Batch norm & 8.05\\
		\hline 
		1-2 & 19. Activation & 3.47\\
		\hline 
		1-3 & 04. Initial learning rate * 05. Initial learning rate (Detail 1) & 3.12\\
		\hline 
		1-4 & 19. Activation * 24. Batch norm & -2.55\\
		\hline 
		1-5 & 04. Initial learning rate & -2.34\\
		\hline 
		1-6 & 28. Weight decay  & -1.90\\
		\hline 
		1-7 & 24. Batch norm * 28. Weight decay & 1.79\\
		\hline 
		1-8&34. Optnet * 35. Share gradInput * 52. Dummy \footnote{This is an interesting feature. In the code repository that we use, optnet, shared gradInput are two special options of the code and cannot be set true at the same time, otherwise the training becomes unpredictable.
		} & 1.54\\
		\hline 
		2-1 & 03. Optimization method & -4.22\\
		\hline 
		2-2 & 03. Optimization method * 10. Use momentum & -3.02\\
		\hline 
		2-3 & 15. Resblock first activation & 2.80\\
		\hline 
		2-4 & 10. Use momentum& 2.19\\
		\hline 
		2-5 & 15. Resblock first activation * 17. Resblock third activation  & 1.68\\
		\hline 
		2-6& 01. Good initialization &-1.26\\
		\hline 
		2-7&01. Good initialization * 10. Use momentum &-1.12\\
		\hline
		2-8&01. Good initialization * 03. Optimization method &0.67\\
		\hline 
		3-1&
		29. Weight decay parameter &-0.49\\
		\hline
		3-2& 28. Weight decay&-0.26\\
		\hline
		3-3& 06. Initial learning rate (Detail 3) * 28. Weight decay &0.23\\
		\hline
		3-4& 25. Batch norm tuning &0.21\\
		\hline
		3-5& 28. Weight decay * 29. Weight decay parameter &0.20\\
		\hline
	\end{tabular}
\end{table}

\subsection{Generalizing from small networks to big networks}
\label{sec:tune_small}

In our experiments, \SHOR{} first runs on a small network to extract important
features and then uses these features to do fine tuning on a big network. Since \SHOR{} 
finds significantly better solutions, it is natural to ask whether other algorithms can also exploit this strategy to improve performance. 

Unfortunately, it seems that all the
other algorithms do not naturally support feature extraction from a small network. For Bayesian Optimization techniques, small networks and large networks have different optimization spaces. Therefore without some modification, Spearmint cannot use information from the small network to update the prior distribution for the large network. 

Random-search-based techniques are able to find configurations with
low test error on the small network, which might be good candidates
for the large network. However, based on our simulation, good
configurations of hyperparameters from random search do not generalize
from small networks to large networks. This is in contrast to
important features in our (Fourier) space, which do seem to generalize. 

To test the latter observation using Cifar-10 dataset, we first spent 7 GPU days on $8$ layer network to find top $10$ configurations among $300$ random selected configurations. Then we apply these $10$ configurations, as well as $90$ locally perturbed configurations (each of them is obtained by switching one random option from one top-$10$ configuration), so in total $100$ ``promising'' configurations, to the large $56$ layer network. 
This simulation takes $27$ GPU days, but the best test error we obtained is 
only $11.1\%$, even worse than purely random search. 
Since Hyperband is essentially a fast version of Random Search, it
also does not support feature extraction. 

Hence, being able to extract important features from small networks seems empirically to be a  unique feature of \SHOR{}.